\documentclass{article}
\usepackage[english]{babel}
\usepackage[letterpaper,top=2cm,bottom=2cm,left=3cm,right=3cm,marginparwidth=1.75cm]{geometry}

\usepackage{authblk}
\usepackage[usenames, dvipsnames]{color}
\usepackage[utf8]{inputenc} 
\usepackage[T1]{fontenc}    
\usepackage[colorlinks, citecolor = Blue]{hyperref}        
\usepackage{url}            
\usepackage{booktabs}       
\usepackage{amsfonts}       
\usepackage{nicefrac}       
\usepackage{microtype}      
\usepackage{xcolor}         

\usepackage{graphicx}
\usepackage{amsmath}
\usepackage{amsthm}
\usepackage{algorithmic}
\usepackage{amssymb}
\usepackage{amsfonts}
\usepackage{subfigure}
\usepackage{float}

\usepackage{makecell}
\usepackage{algorithm}
\usepackage{framed}
\usepackage{booktabs}
\usepackage{multicol}
\usepackage{algorithmic}

\newtheorem{theorem}{Theorem}

\newtheorem{lemma}[theorem]{Lemma}

\newtheorem{definition}[theorem]{Definition}

\newtheorem{fact}[theorem]{Fact}

\begin{document}

\title{Tight Memory-Regret Lower Bounds for \\ Streaming Bandits}

\author[1]{Shaoang Li}
\author[1]{Lan Zhang}
\author[1]{Junhao Wang}
\author[1]{Xiang-Yang Li}

\affil[1]{University of Science and Technology of China}

\affil[ ]{\{lishaoa, junhaow\}@mail.ustc.edu.cn,\{zhanglan, xiangyangli\}@ustc.edu.cn}
\date{}

\maketitle

\begin{abstract}
In this paper, we investigate the streaming bandits problem, wherein the learner aims to minimize regret by dealing with online arriving arms and sublinear arm memory. 
We establish the \emph{tight} worst-case regret lower bound of $\Omega \left( (TB)^{\alpha} K^{1-\alpha}\right), \alpha = 2^{B} / (2^{B+1}-1)$ for any algorithm with a time horizon $T$, number of arms $K$, and number of passes $B$.  
The result reveals a separation between the stochastic bandits problem in the classical centralized setting and the streaming setting with bounded arm memory. 
Notably, in comparison to the well-known $\Omega(\sqrt{KT})$ lower bound, an additional double logarithmic factor is unavoidable for any streaming bandits algorithm with sublinear memory permitted.
Furthermore, we establish the first  instance-dependent lower bound of $\Omega \left(T^{1/(B+1)} \sum_{\Delta_x>0} \frac{\mu^*}{\Delta_x}\right)$ for streaming bandits.  
These lower bounds are derived through a unique reduction from the regret-minimization setting to the sample complexity analysis for a sequence of $\epsilon$-optimal arms identification tasks, which maybe of independent interest. 
To complement the lower bound, we also provide a multi-pass algorithm that achieves a regret upper bound of $\tilde{O} \left( (TB)^{\alpha} K^{1 - \alpha}\right)$ using constant arm memory.
\end{abstract}

\section{Introduction}
\label{sec:introduction}
The multi-armed bandits framework is a classic paradigm for sequential prediction and decision-making, encompassing applications in various domains including medical trials, web search, recommendation systems, and online advertising. 
The goal of the learner is to maximize the expected cumulative reward by selecting arms (actions, options, decisions) over a sequence of trials. 
In bandits problem, there are $T$ rounds and $K$ arms, each of which is associated with an unknown reward distribution. 
In each round, the algorithm selects an arm from the available candidates and subsequently observes a reward that is independently sampled from the corresponding distribution. 
The algorithm aims to minimize the regret, which quantifies the additional cumulative loss incurred by the algorithm compared to always playing the optimal arm.

Several well-known algorithms, such as UCB1 and Thompson Sampling, have been  proven to achieve near-optimal performance while maintaining a comprehensive record of all arms in memory. However, in scenarios involving limited memory or a large action space, the storage of all arms becomes problematic for the learner. 
An exemplary case arises in recommendation systems, where the learner must choose from millions of items, such as music and movies, to present to users with the objective of maximizing the click-through rate. Accomplishing this task using only a sublinear number of items stored in memory poses a considerable challenge. 
Motivated by practical applications, substantial efforts have been dedicated to the development of algorithms for the streaming bandits setting with constrained $M$ ($M<K$) arm memory \cite{DBLP:conf/stoc/AssadiW20,DBLP:journals/corr/abs-2012-05142,DBLP:conf/colt/AgarwalKP22}. 
In this streaming version, arms arrive one at a time in a sequential stream, and the algorithm can only select arms that are stored in memory.  
Once the memory is full and the algorithm  intends to choose a new arm, it must remove at least one existing arm from memory before selecting the new one. Consequently, all associated statistics, including the arm's index, mean reward, and number of pulls, are discarded.  
The streaming setting necessitates the use of more sophisticated strategies to effectively balance the reading of new arms from the stream and retrieving arms from memory, thereby optimizing the learner's reward.

In recent results, \cite{DBLP:conf/aistats/LiauSPY18} consider stochastic streaming bandits with constant arm memory and propose algorithms aimed at minimizing regret through multiple passes over the stream. 
\cite{DBLP:conf/aaai/ChaudhuriK20} investigate stochastic bandits with $M$ stored arms and multiple passes, and provide an algorithm with regret $\Tilde{O}(KM + (K^{3/2}\sqrt{T})/M)$. 
For the single-pass streaming bandits model, \cite{DBLP:journals/corr/abs-2012-05142} provide the 
$\Omega(K^{1/3}T^{2/3}/M^{7/3})$ worst-case lower bound and 
$\tilde{O}(K^{1/3}T^{2/3})$ upper bound. 
Recently, \cite{DBLP:conf/colt/AgarwalKP22} explore the dependence on the number of passes $B$, and establish the $\Omega(4^{-B}T^{2^B/(2^{B+1}-1)})$ worst-case lower bound and $\Tilde{O}( T^{2^B/(2^{B+1}-1)} \sqrt{KB} )$ regret upper bound for algorithms with $B$ passes and $o(K/B^2)$ arm memory. 
Concurrent to our work, \cite{DBLP:journals/corr/abs-2306-02208} establish tight regret bounds for single-pass streaming bandits. 
However, little effort has been made to comprehend the instance-dependent regret lower bound and the dependency on the number of arms $K$ for the worst-case regret lower bound, despite the fact that the presence of a large action space is the primary motivation for this line of research.

\begin{table*}[t]
    \centering
    \caption{A summary of our results for streaming bandits with time horizon $T$, number of arms $K$, number of passes $B$, bounded arm memory $M<K$, and $\alpha = 2^{B} / (2^{B+1}-1) $. Previous lower bounds \cite{DBLP:conf/colt/AgarwalKP22} hold with $M = o(K/B^2)$ while ours hold with $M = o(K/B)$.}
    \vspace{0.08in}
    \label{tab:comm_bound}
    \resizebox{1\linewidth}{!}{
    \begin{tabular}{|c | c | c | c | c|}
    
    \hline
    \rule{0pt}{15pt}
         & \makecell[c]{\textbf{Previous LB} }  & \textbf{Previous UB} & \makecell[c]{\textbf{Our LB}\\(Theorem \ref{the:worst} and \ref{the:instance})}& \makecell[c]{\textbf{Our UB}\\(Theorem \ref{the:upper})}\\
         
    \hline
    \rule{0pt}{15pt}
        \makecell[c]{1-pass algorithm} & \makecell[c]{$\Omega(K^{1/3}T^{2/3})$ \\ \cite{DBLP:journals/corr/abs-2306-02208} } & \makecell[c]{ $O(K^{1/3}T^{2/3})$ \\ \cite{DBLP:journals/corr/abs-2306-02208} } & $\Omega(K^{1/3}T^{2/3})$ & $\Tilde{O}(K^{1/3}T^{2/3})$\\ 
    \hline
    \rule{0pt}{18pt}
        \makecell[c]{$\log \log T$-pass\\ algorithm} & \makecell[c]{$\Omega \left( 4^{-\log\log T} \sqrt{T} \right)$ \\ \cite{DBLP:conf/colt/AgarwalKP22} } & \makecell[c]{$\Tilde{O}(\sqrt{KT})$ \\ \cite{DBLP:journals/corr/abs-2112-06130} } & $\Omega(\sqrt{KT \log \log T})$  & $\Tilde{O}(\sqrt{KT})$ \\ 
    \hline
    \rule{0pt}{20pt}
        \makecell[c]{$B$-pass algorithm} & \makecell[c]{$\Omega \left( 4^{-B} T^{\alpha} \right) $ \\ \cite{DBLP:conf/colt/AgarwalKP22}  } & \makecell[c]{$\Tilde{O} \left( T^{\alpha} \sqrt{KB }\right)$ \\ \cite{DBLP:conf/colt/AgarwalKP22}} & $\Omega \left( (TB)^{\alpha} K^{1 - \alpha}\right)$ & $\ \,\Tilde{O} \left( (TB)^{\alpha} K^{1 - \alpha}\right)\ \,$\\
    \hline
    \rule{0pt}{20pt}
        \makecell[c]{Instance-dependent\\ bound} & N/A & \makecell[c]{$\Tilde{O} \left(  (B + T^{\frac{1}{B+1}}) \sum\limits_{\Delta_x>0} \frac{1}{\Delta_x}\right)$ \\ \cite{DBLP:conf/colt/AgarwalKP22}  } & $\Omega \left( T^{\frac{1}{B+1}} \sum\limits_{\Delta_x>0} \frac{\mu^*}{\Delta_x} \right)$ & N/A\\
    \hline
    \end{tabular}
    } 
    \vspace{-0.1in}
\end{table*}

\paragraph{Our Contributions.}
In this work, we revisit the streaming bandits with limited memory problem in an effort towards a better understanding of memory-regret trade-off and mutual information for different passes. 
We present \emph{tight} worst-case regret lower bound and instance-dependent regret lower bound for general stochastic streaming bandits with sublinear arm memory. 
We provide a delicate model of how the regret lower bound is influenced by the bounded arm memory $M$, the number of arms $K$, the number of passes $B$, the time horizon $T$, and the "constant" which is solely dependent on the problem instance. 
To complement the lower bound, we devise a multi-pass algorithm that utilizes only constant arm memory and achieves near-optimal regret up to a logarithmic factor. 

For the worst-regret lower bound of the streaming bandits, we establish that any algorithm using $1 \leq B \leq \log\log T$ passes over the stream and $o(K/B)$ arm memory incurs regret $\Omega \left( (TB)^{\alpha} K^{1-\alpha}\right)$, where $\alpha = 2^{B} / (2^{B+1}-1) $. 
We present some results with special $B$: any 1-pass algorithm with $o(K)$ bounded memory incurs $\Omega(K^{1/3}T^{2/3})$ regret; any $\log \log T$-pass algorithm with $o(\frac{K}{\log \log T})$ arm memory incurs $\Omega(\sqrt{KT\log \log T})$ regret. 
Our results demonstrate a separation for stochastic multi-armed bandits problem between the classical centralized setting and the streaming setting when only $B$ passes and $o(K/B)$ bounded arm memory are permitted. 
In comparison to the $\Omega(\sqrt{KT})$ worst-case lower bound for classical bandits, our results reveal that any streaming bandits algorithm with bounded arm memory incurs an additional double logarithmic factor.  
Furthermore, we establish the first instance-dependent lower bound, that any algorithm achieves regret $\Omega \left(T^{1/(B+1)} \sum_{\Delta_x>0} \frac{\mu^*}{\Delta_x}\right)$ with $o(K/B)$ bounded arm memory, where $\Delta_x := \mu^* - \mu_x$ is the gap of arm $x$ with respect to the highest mean. 
Compared to the logarithmic regret bound for the centralized setting, we highlight that an additional regret factor of $O(T^{1/(B+1)})$ is unavoidable due to the  adversarial stream and bounded arm memory. 

To complement the lower bounds, we provide an upper bound for the streaming bandits with multiple passes and bounded arm memory. 
Theoretical analysis of the lower bound implies that, during the $p$-th pass, the number of pulls for any arm should be capped by $\Theta( (TB)^{2\beta_p}K^{-2\beta_p}), \beta_p = \frac{2^{B-p} (2^{p}-1)}{2^{B+1}-1}$.
This crucial idea serves as the foundation for our algorithm and may be of independent interest. 
We show that given a time horizon $T$, a stream consisting of $K$ arms, and $B$ passes over this stream, there exists an algorithm capable of achieving regret $\textnormal{Regret}(\mathcal{A}) = O \left( (TB)^{\alpha} K^{1-\alpha} \sqrt{\log T}\right)$, while only storing two arms in memory. 
Note that the algorithm achieves near-optimal regret up to a logarithmic factor. 
In comparison to the result $O(T^{\alpha} \sqrt{KB \log T})$ in \cite{DBLP:conf/colt/AgarwalKP22}, our upper bound demonstrates a more tight dependence on both $K$ and $B$.

\paragraph{Notations.}
In this paper, we adopt the convention of using bold fonts to represent vectors and matrices. 
For a positive integer $T$, we use $[T]$ to denote the set $\{1, 2, \ldots, T\}$. 
Regarding a set $\mathcal{K}$, the notation $|\mathcal{K}|$ is utilized to denote its cardinality.  
Given a random variable $X$, we use $\mathbb{E}[X]$ for its expectation, standard notions of $H(X)$ to denote its entropy, and $H(X|Y)$ for the conditional entropy of random variable $X$ given random variable $Y$. 
The mutual information is denoted by $I(X;Y) = H(X)-H(X|Y)$.
For an event $\mathcal{E}$, the probability of $\mathcal{E}$ is denoted as $\mathbb{P}[\mathcal{E}]$.


\section{Problem Setup}
\label{sec:problem}
Given $T$ rounds and a finite set $\mathcal{K}$ of ($K = |\mathcal{K}|$) arms, they are indexed by $[T] = 1, 2, \ldots , T$, $[K] = 1, 2, \ldots, K$, respectively. 
Each arm $x \in \mathcal{K}$ is associated with an unknown reward distribution $\nu_x$ defined over $[0,1]$ with a fixed unknown mean $\mu_x$. 
Let $x^*$ and $\mu^*:= \max_{x \in \mathcal{K}} \mu_x$ denote the optimal arm and its expected per-round reward, respectively. 
In each round $t \in [T]$, the algorithm plays an arm $x_t \in \mathcal{K}$  and subsequently observes a scalar-valued reward $r_t$, which is independently sampled from the corresponding reward distribution. 
For the regret minimization setting, the algorithm $\mathcal{A}$ aims to minimize the cumulative regret, which is defined as $\textnormal{Regret}(\mathcal{A}) :=  \sum_{t=1}^T \left( \mu^* - \mu_{x_t} \right)$.

Let $\Delta_x := \mu^* - \mu_x$ denote the expected sub-optimality (gap) of arm $x$ with respect to the highest mean. 
Let $\Delta_{[i]}, i \in [K]$ denote the gap of the $i$-th optimal arm. 
For $0 \leq \epsilon_1 < \epsilon_2 \leq 1$, assume $\epsilon_1 = o(\epsilon_2)$ and consider the set of arms whose gap satisfy $\mathcal{K}^{(\epsilon_1,\epsilon_2)} = \{ x\in \mathcal{K}: \epsilon_1 \leq \Delta_x < \epsilon_2 \}$. 
For the $\epsilon$-optimal arm identification problem, the leaner seeks to identify at least one arm belonging to the set $\mathcal{K}^{(0, \epsilon)}$ with high probability while minimizing the total arms pulls.

We briefly describe the streaming model and memory model below. 
We model the set of arms as a read-only stream in arbitrary (possibly adversarial) order.  
The algorithm is allowed to play at most $B$ passes over the stream. 
While the order of the arms may vary between passes, the set of arms remains the same for each pass.
The algorithm is only allowed to pull arms that are stored in memory.
When the memory becomes full and the algorithm attempts to select a new arm, it must discard at least one arm from memory.
All information pertaining to the discarded arm, such as its index, mean reward, and number of pulls, is forgotten and will not be retained. 
The space complexity of the algorithm is evaluated based on the hard constraint on the number of arms that can be stored in memory.
This aligns with the assumption of having oracle access to the input arm, which is a standard definition in streaming problems.
The space measure described in this paper can be easily mapped to the space complexity in the RAM model.

There is an alternative setting in which the learner is allowed to permanently remove arms from the stream, ensuring they do not appear in subsequent passes.
For example, certain arms may be discarded if they are deemed to be strictly sub-optimal.
It is worth mentioning that our lower and upper bounds remain applicable in this setting as well.

\section{Worst-Case Lower Bound}
\label{sec:worst}
In this section, we present an overview of the worst-case lower bound and provide the hard distribution employed to establish it. 
The methodologies and analyses used in constructing are re-utilized in the subsequent section for the instance-dependent lower bound.

\begin{theorem}
\label{the:worst}
    Given a time horizon $T$, a stream of $K$ arms, and $1 \leq B \leq \log\log T$ passes over this stream, consider any algorithm $\mathcal{A}$ that only store $o(K/B)$ arms in memory, we have
    \begin{equation*}
        \textnormal{Regret}(\mathcal{A}) = \Omega \left( (TB)^{\frac{2^B}{2^{B+1}-1}} K^{\frac{2^B-1}{2^{B+1}-1}}\right)
    \end{equation*}
    in expectation for some problem instances. 
\end{theorem}

We provide an overview of the lower bound. At a high level, any algorithm faces two dilemmas during the learning process.
The first dilemma arises within each single-pass due to the memory constraint that only a limited number of arms can be selected at any given time. The algorithm has to decide whether to conduct adequate exploration over the arms it has in memory or to rapidly advance in the stream. Sufficient exploration of arms in memory may result in a large regret when all arms encountered in the stream have large reward losses, and no arms with large rewards have yet been encountered. Insufficient exploration results in the learner discarding arms in memory after a few samples and attempting to move ahead quickly over the stream, increasing the uncertainty and the risk of missing potentially optimal arms due to lack of information. The trade-off between these two cases in a single pass has been explored in both the pure exploration setting \cite{assadi2022singlepass} and regret minimization setting \cite{DBLP:journals/corr/abs-2012-05142, DBLP:conf/colt/AgarwalKP22}. 
We revisit this dilemma and provide a more nuanced perspective. 

Another dilemma arises from the allocation of pulls among passes, which has not been previously discussed in the literature. Spending more time during one pass captures more meaningful information, which could decrease uncertainty about the optimal arm and help in discarding sub-optimal arms quickly for all subsequent passes, but with more incurred regret for the ongoing pass due to more pulls of sub-optimal arms. On the other hand, fewer pulls provide little meaningful information about the optimal arm, leaving the learner essentially with the same hard problem. In other words, the learner must decide whether to take the risk of acquiring more information for the future or do a few samples prudently. A natural idea is to do the allocation in a way that depends on the history. We provide an analysis of the best possible strategy that one can hope to achieve. 

One of our key ideas is introducing the reduction from the regret minimization setting to a series of $\epsilon$-optimal arms identification problems, where each corresponds to one pass during the learning process. 
Observe that any algorithm achieving $O(\epsilon T)$ worst-case regret can also output at least one $\epsilon$-optimal arm with a probability of $1-o(1)$ for any problem instance, simply by selecting one arm randomly from the pull history. 
We prove that a lower bound of $\Omega(\epsilon^{-2})$ samples is necessary for any algorithm to distinguish between an arm with a reward distribution of $\mu$ and an arm with a reward distribution of $\mu+\epsilon$, by using techniques of distributed data processing inequality. 
We then establish the $\Omega(K\epsilon^{-2})$ sample complexity lower bounds by constructing a hard distribution, both for the $\epsilon$-optimal arm detection setting and the $\epsilon$-optimal arm identification setting, with the single-pass data stream and bounded arm memory. 

Note that the incurred regret during  exploration for any single-pass $\epsilon$-optimal arm identification algorithm is at least $\Omega(\sum_{x \in \mathcal{K}} \Delta_x \epsilon^{-2}  )$, while the incurred regret is $\Omega(\epsilon(T-K\epsilon^{-2}))$ during exploitation with the $\epsilon$-optimal arms. 
With the knowledge of one $\epsilon'$-optimal arm (assuming $\epsilon' = \omega(\epsilon)$), then the exploration regret lower bound for the single-pass algorithms could be decreased to $\Omega(\sum_{x \in \mathcal{K}^{(0,\epsilon')}}\Delta_x \epsilon^{-2} + \sum_{x \in \mathcal{K}^{(\epsilon', 1)}}\Delta_x (\epsilon')^{-2})$. 
The knowledge could be gathered by a single-pass $\epsilon'$-optimal arm identification algorithm. 
Based on this, we can devise a two-pass algorithm where the learner carefully selects the parameters $\epsilon$ and $\epsilon'$ to minimize the total regret. This thought experiment can be extended to the multi-pass case.
We provide an algorithm to construct a hard distribution for the multi-pass case building upon single-pass hard distributions.  The choice of hard distribution is crucial to ensure that any learner nevertheless suffers large regret during the learning process. 

To formalize the ideas presented above, we begin by providing the information-theoretic results for the distributed distribution detection problem and the single-pass $\epsilon$-optimal arms identification problem. These results serve as the foundation for our analysis of multi-pass case. 

\subsection{Lower Bound for $\epsilon$-Optimal Arms Identification}
We provide the sample complexity lower bound for the $\epsilon$-optimal arms identification problem, both considering the detection setting and the identification setting. 
We utilize the theoretical result of the distribution detected problem to construct hard distribution.

\begin{definition}[Distributed detection problem]
    For fixed distributions $\mu_0$ and $\mu_1$, let $X^{(1)}, \ldots , X^{(n)}$ be sampled i.i.d. from $\mu_V$, for $V \in \{0,1\}$. 
    The distributed detection problem is the task of determining whether $V=0$ or $V=1$, given the values of $\boldsymbol{X} = X^{(1)}, \ldots , X^{(n)}$.
\end{definition}
Then we provide a lower bound for the entropy of the error bound of any estimator, which relies on strong data processing inequality and Fano's inequality.

\begin{lemma}
\label{lem:distri}
    For the distributed detection problem, assume we sample $V$ from a distribution. 
    Suppose $\frac{1}{2} \leq \mu_0 < \mu_1 \leq \frac{1}{2} + c$ for some constant $0 < c <\frac{1}{2}$ and let $\Delta: \mu_1 - \mu_0$ be their gap. 
    Let $\hat{V} := g(\boldsymbol{X})$ be an estimate of $V$ from observing $\boldsymbol{X} = \{ X^{(1)}, \ldots , X^{(n)} \}$, where $g$ is an estimation function. Define the probability of errors as $P_e := \mathbb{P}(\hat{V} \neq V)$. 
    Then the following inequality holds
    \begin{equation*}
        H(P_e) \geq H(V) - O(n \cdot \Delta^{2}). 
    \end{equation*}
\end{lemma}

The details of the proof are provided in Appendix \ref{subsec:distri}. 
The result indicates the information-theoretic hardness in distinguishing distributions with distinct values, and provides a refined model that captures the intricate relationship between the error probability $P_e$, the Bernoulli random variable representing the distribution across the parameter space $V$, the sample size $n$, and the value gap $\Delta$. 
Building upon Lemma \ref{lem:distri}, which addresses the distributed detection problem, the subsequent result applies to the bandits setting with $K$ arms.

\begin{lemma}[The detection setting]
\label{lem:dist}
    Given an arm set $\mathcal{K}$ and two Bernoulli distributions $\mu_0, \mu_1 \in [\frac{1}{2}, \frac{3}{4}]$ and assume $\mu_0 < \mu_1$. 
    Let $\epsilon$ be their gap.
    The learner aims to distinguish between two cases. 
    \begin{enumerate}
        \item (NO case, `$V=0$') The rewards for all arms are chosen from the distribution $\mu_0$. 
        \item (YES case, `$V=1$') One index $i^* \in \mathcal{K}$ is selected arbitrarily. 
        The reward of the arm with index $i^*$ is chosen from the distribution $\mu_1$ and rewards of others are chosen from the distribution $\mu_0$, respectively.
    \end{enumerate}
    Any algorithm that solves this problem with constant probability at least $1-q, q \in [0, 0.5)$ has to use at least $\Omega(K\epsilon^{-2})$ arm pulls. 
\end{lemma}

We provide the details of the proof in Appendix \ref{subsec:dist}. 
Lemma \ref{lem:dist} provides the sample complexity lower bound for detecting the existence of one unique $\epsilon$-optimal arm, even with unbounded arm memory. 
In the YES case, there is one unique optimal arm and all others with reward gap $\epsilon$, whereas in the NO case, rewards for all arms are randomly sampled from the same distribution. 
Note that this lower bound applies not only to worst-case streams but also to random-order streams and i.i.d. streams. 

Then we extend this example to the single-pass streaming setting of $\epsilon$-optimal arms identification problem. 
We concentrate on the lower bound for deterministic algorithms over a hard distribution input, which is associated with an identical lower bound for randomized algorithms by Yao's minimax principle \cite{yao1977probabilistic}. 
We define the hard distribution below. 

\begin{framed}
    \textbf{Distribution} $\mathcal{D}^{(\mathtt{base},\epsilon)}_{\mathcal{K}}$: Given a set of arms $\mathcal{K} = \{x_1, x_2, \ldots, x_K\}$ and  parameters $\mathtt{base}, \epsilon$, we construct the hard distribution as follows:
    \begin{enumerate}
        \item Sample one index $\Tilde{i}$ uniform at random from set $[K-1]$. 
        \item Let $\mu_{\Tilde{i}} = \mathtt{base} + \epsilon$ for the selected arm $x_{\Tilde{i}}$, and let $\mu_x = \mathtt{base}$ for all other arms $x_i$ that satisfy $i \in [K-1]$ and $i \neq \Tilde{i}$. 
        \item Sample $\mathtt{coin}$ from an independent fair coin. For the last arm, let $\mu_{x_K} = \mathtt{base}$ if $\mathtt{coin} = 0$ and let $\mu_{x_K} = \mathtt{base} + 2\epsilon$ if $\mathtt{coin} = 1$.
    \end{enumerate}
\end{framed}

For the hard adversarial distribution, the expected rewards of all arms are greater than or equal to $\mathtt{base}$. 
The first $K-1$ arms of the stream correspond to the YES case in Lemma \ref{lem:dist}, and the last arm is the unique $\epsilon$-optimal arm with probability $\frac{1}{2}$. 
Previous work provides the special case of the construction with $\mathtt{base} = \frac{1}{2}$ for the single-pass streaming bandits \cite{DBLP:journals/corr/abs-2012-05142} and the single-pass optimal arm identification setting \cite{assadi2022singlepass}, respectively. 
Our analysis relies on the result of Lemma \ref{lem:distri} and the reduction from distributed distribution detection problem to solving the hard distribution described above, and the hard distribution also serves as building blocks for the hard distribution for multi-passes. 
The following lemma provides a sample complexity lower bound of streaming $\epsilon$-optimal arms identification setting. 

\begin{lemma}[The identification setting]
\label{lem:dis}
    There always exists a family of streaming stochastic bandits instances, such that any single-pass algorithm using $o(K)$ arm memory that identifies at least one $\epsilon$-optimal arm with constant probability $1-q, q \in [0, 0.5)$ has to use $\Omega(\epsilon^{-2})$ pulls for each of an arm set with $\Omega(K)$ arms.
\end{lemma}

We defer the details of the proof in Appendix \ref{subsec:dis}. 
Lemma \ref{lem:dis} shows that $\Omega(K\epsilon^{-2})$ pulls are also necessary for the identification setting and provide a more delicate result that at least $\Theta(K)$ number of arms have to be sampled at least $\Omega(\epsilon^{-2})$ times.

\subsection{The Hard Distribution for the Worst-Case Lower Bound}

\begin{algorithm}[t]
   \caption{The algorithm to construct hard distribution}
   \label{alg:construct}
   {\bfseries Input:} a set of arms $\mathcal{K}$, number of passes $B$, a series of parameters $ \epsilon_1, \epsilon_2, \ldots, \epsilon_B $, where $\epsilon_p \leq \frac{1}{16}, p \in [B]$.
\begin{algorithmic}[1]
   \STATE $\mathtt{base}_1 = \frac{1}{2}$, division $\mathcal{K}_1, \mathcal{K}_2, \ldots, \mathcal{K}_B$.
   \STATE Let $\mathcal{D}_1$ be distribution $\mathcal{D}_{\mathcal{K}_1}^{(\mathtt{base_1}, \epsilon_1)}$.
   \FOR{$p=2, 3, \ldots, B$}
   \STATE Sample $\mathtt{coin}$ from an independent fair coin.
   \STATE Let $\mathtt{base}_p = \mathtt{base}_{p-1}$ if $\mathtt{coin} = 0$ and let $\mathtt{base}_p = \mathtt{base}_{p-1} + 2\epsilon_{p-1}$ if $\mathtt{coin} = 1$.
   \STATE Let $\mathcal{D}_p$ be distribution $\mathcal{D}_{\mathcal{K}_p}^{(\mathtt{base}_p, \epsilon_p)}$.
   \ENDFOR
    \STATE Output hard distribution $\mathcal{D}_1 \otimes \mathcal{D}_2 \otimes \cdots \otimes \mathcal{D}_B$.
\end{algorithmic}
\end{algorithm}

We can now show the hard distribution for the worst-case regret lower bound with multi-pass and bounded arm memory. 
The process to construct the hard multi-pass distribution was shown in Algorithm \ref{alg:construct}, and the proof of Theorem \ref{the:worst} is built on it. 
We provide the main idea below and defer the details in Appendix \ref{subsec:worst}. 
The Algorithm \ref{alg:construct} provide a paradigm of the hard distribution construction for streaming bandits.

Note that Lemma \ref{lem:dis} also implies that the incurred regret during the exploration for any single-pass $\epsilon$-optimal arm identification algorithm is at least $\Omega(\sum_{x \in \mathcal{K}} \Delta_x \epsilon^{-2}  )$ for the single-pass hard distribution resulting from the adversary. 
Another remark is that selecting $\epsilon$ parameters is meaningless and wasteful for obtaining information of the optimal arm. 
Consider the scenario in which the single-pass algorithm uses $\epsilon_1^{-2}$ pulls for all arms $x \in \mathcal{K}_1$, whereas uses $\epsilon_2^{-2}$ pulls for all arms $x \in \mathcal{K}_2$:
\begin{enumerate}
    \item If $\epsilon_1 = \Theta(\epsilon)$ and $\epsilon_2 = \Theta(\epsilon)$, then we have $|\mathcal{K}_1 \cup \mathcal{K}_2| = \Theta(K)$ and the algorithm has to use $\Omega(\epsilon^{-2})$ pulls for each of $\Theta(K)$ number of arms identify at least one $\epsilon$-optimal arm with constant probability. 
    \item If $\epsilon_1 = \omega(\epsilon)$ and $\epsilon_2 = \omega(\epsilon)$, then the exploration is insufficient, and the algorithm cannot identify at least one $\epsilon$-optimal arm with constant probability. 
    \item If $\epsilon_1 = \omega(\epsilon)$ and $\epsilon_2 = \Theta(\epsilon)$, then we have $|\mathcal{K}_2| = \Theta(K)$ and the algorithm has to use $\Omega(\epsilon^{-2})$ pulls for each of $\Theta(K)$ number of arms. 
    We have similar result for the case with $\epsilon_2 = \omega(\epsilon)$ and $\epsilon_1 = \Theta(\epsilon)$. 
\end{enumerate}
This thought experiment could be extended to cases with multiple parameters.
In the remainder of this section, we will only consider algorithms that select one parameter per pass.

Then we going to focus on the interactions between different passes and start with the two-pass case. 
Recall that the regret lower bound of the single-pass algorithm is $\Omega(\sum_{x \in \mathcal{K}} \Delta_x \epsilon^{-2}  )$. 
With the knowledge of one $\epsilon'$-optimal arm, any algorithm still has to use $\Omega((\epsilon')^{-2})$ samples to identify $\epsilon'$-optimal arms and use $\Omega((\epsilon)^{-2})$ samples to identify $\epsilon$-optimal arms, then the exploration regret lower bound for the single-pass algorithms could be reduced to $\Omega(\sum_{x \in \mathcal{K}^{(0,\epsilon')}}\Delta_x \epsilon^{-2} + \sum_{x \in \mathcal{K}^{(\epsilon', 1)}}\Delta_x (\epsilon')^{-2})$. 
And the knowledge could be collected by a single-pass $\epsilon'$-optimal arm identification algorithm. 
Similar to the thought experiment for different parameters within single-pass, the case of $\epsilon = \Omega(\epsilon')$ is trivial. 
Another leading result is that any algorithm that fails to gather enough information about the $\epsilon'$-optimal arm during the $\epsilon'$-pass will require much more samples during the subsequent phases (for example, $\epsilon$-pass), due to the insufficient knowledge that leaked from the previous passes.
Then we obtain a two-pass algorithm $\mathcal{A}_{\textnormal{two-pass}}$ and the learner has to select parameters $\epsilon$ and $\epsilon'$ carefully to minimize the total regret. 
The exploration regret satisfies the following inequality with constant $c_1$:
\begin{equation*}
\vspace{-0.05in}
    \textnormal{Explore-Regret}(\mathcal{A}_{\textnormal{two-pass}}) \geq c_1 \cdot \Big( \underbrace{ \sum_{x \in \mathcal{K}}\Delta_x (\epsilon')^{-2} }_{\epsilon'-\textnormal{pass}} +  \underbrace{ \sum_{x \in \mathcal{K}^{(0,\epsilon')}}\Delta_x \epsilon^{-2} + \sum_{x \in \mathcal{K}^{(\epsilon', 1)}}\Delta_x (\epsilon')^{-2} }_{\epsilon-\textnormal{pass}}  \Big).
\end{equation*}
\vspace{-0.05in}

Then we extend the result to the multi-pass case. 
Let the $p$-th pass be with parameter $\epsilon_p$, i.e., aims to identify at least one $\epsilon_p$-optimal arm during the $p$-th pass. 
Similar to the two-pass case, the case with $\epsilon_{p+1} = \Omega(\epsilon_{p})$ is trivial for all $p \in [B-1]$. 
This also implies that with enough knowledge about the $\epsilon_p$-optimal arm, the $\epsilon_{p+1}$-optimal arm detection setting, even the detection setting is still hard for the learner. 
For the regret within the $\epsilon_{p+1}$-pass, we have the following inequality
\vspace{-0.05in}
\begin{equation*}
    \max_{\Delta_x, x \in \mathcal{K}} \Big( \sum_{x \in \mathcal{K}^{(0,\epsilon_p)}}\Delta_x \epsilon_{p+1}^{-2} + \sum_{x \in \mathcal{K}^{(\epsilon_p, 1)}}\Delta_x (\epsilon_p)^{-2} \Big) \geq \max_{\Delta_x, x \in \mathcal{K}^{(0,\epsilon_p)}} \sum_{x}\Delta_x \epsilon_{p+1}^{-2} \geq \frac{1}{2} |\mathcal{K}^{(\epsilon_{p+1}, \epsilon_p)}| \epsilon_p \epsilon_{p+1}^{-2}. 
\end{equation*}
In other words, even with the knowledge from the previous passes, $\Omega \left( \sum_x \epsilon_p \epsilon_{p+1}^{-2} \right)$ regret is inevitable for the $\epsilon_{p+1}$-pass while the adversary selects the hard distribution carefully.

Now we provide the hard distribution for the multi-passes case. 
Note that an algorithm with $B$ passes and $M$ bounded arm memory could be denoted as $\langle B, M, \epsilon_1, \epsilon_2, \ldots, \epsilon_B \rangle$. 
Let $\mathcal{K}_1, \mathcal{K}_2, \ldots, \mathcal{K}_B$ be a division of the arm set $\mathcal{K}$. 
Given parameters $\epsilon_1, \epsilon_2, \ldots, \epsilon_B$, we provide an algorithm (shown in Algorithm \ref{alg:construct}) to construct hard distribution. 
Let $\epsilon_0 = \frac{1}{4}$, then the regret of any algorithm satisfies
\vspace{-0.05in}
\begin{equation*}
    \textnormal{Regret}(\mathcal{A}) \geq \frac{c_1}{2} \cdot \mathbb{E} \Bigg[  \min_{\epsilon_1, \epsilon_2, \ldots, \epsilon_B} \max_{\mathcal{K}_1, \mathcal{K}_2, \ldots, \mathcal{K}_B} \underbrace{  \sum_{p = 1}^{B} \frac{\epsilon_{p-1} |\mathcal{K}^{(\epsilon_p, \epsilon_{p-1})}| }{\epsilon_{p}^2} }_{\textnormal{Exploration part}} + \underbrace{ \epsilon_B \left( T - \sum_{p = 1}^{B} \frac{1}{\epsilon^2_p}\right) }_{\textnormal{Exploitation part}} \Bigg]. 
\end{equation*}

For the hard distribution in Algorithm \ref{alg:construct}, the optimal arm falls within the set $\mathcal{K}_p$ with probability at least $1/2^p$. 
In other words, there is at least one hard distribution such that $\mathbb{E}(|\mathcal{K}^{(\epsilon_p, \epsilon_{p-1})}|) \geq \Omega( |\mathcal{K}_p|)$ for all $p \in [B]$. 
Let $|\mathcal{K}_p| = K/B$, we obtain
\vspace{-0.05in}
\begin{equation}
\label{eq:regret_main}
\begin{aligned}
    \textnormal{Regret}(\mathcal{A}) &\geq \frac{c_1}{2} \min_{\epsilon_1, \epsilon_2, \ldots, \epsilon_B} \sum_{p = 1}^{B}  \frac{\epsilon_{p-1}  } {\epsilon_{p}^2} \cdot \frac{K}{B}+ \epsilon_B \left( T - \sum_{p = 1}^{B} \frac{1}{\epsilon^2_p}\right)\\
    &\geq \frac{c_1}{4} \min_{\epsilon_1, \epsilon_2, \ldots, \epsilon_B} \sum_{p = 1}^{B}  \frac{\epsilon_{p-1}  } {\epsilon_{p}^2} \cdot \frac{K}{B}+ \epsilon_B T. 
\end{aligned}
\end{equation}
The second inequality comes from the fact that we have $\epsilon_p = \omega(\epsilon_B)$ for all $p \in [B-1]$. 
Let $f:[0, \frac{1}{4}]^B \rightarrow \mathbb{R}_+$ be the function of $\epsilon_p, p \in [B]$:
\vspace{-0.05in}
\begin{equation*}
    f(\epsilon_1, \epsilon_2, \ldots, \epsilon_B) = \sum_{p = 1}^{B}  \frac{\epsilon_{p-1}  } {\epsilon_{p}^2} \cdot \frac{K}{B}+ \epsilon_B T. 
\end{equation*}
By taking $\frac{\partial f }{\partial \epsilon_1} = \frac{\partial f }{\partial \epsilon_2} = \ldots = \frac{\partial f }{\partial \epsilon_B} = 0$ and let $c_2$ be a constant, we obtain the following result:
\vspace{-0.05in}
\begin{equation}
\label{eq:func_main}
    f(\epsilon_1, \epsilon_2, \ldots, \epsilon_B) \geq c_2 \cdot (TB)^{\frac{2^B}{2^{B+1}-1}} K^{\frac{2^B-1}{2^{B+1}-1}}. 
\end{equation}
Combine \eqref{eq:regret_main} and \eqref{eq:func_main} together, we complete the proof of the worst-case lower bound. 
Note that the function value get the minimum when $B = \Theta (\log \log T)$, so the case with $B = \omega( \log \log T)$ is trival for the worst-case bound. 
Another byproduct is that we obtain $\epsilon_p = \Theta( (TB)^{-\beta_p}K^{\beta_p})$ for the $p$-th pass to minimize the worst-case regret, where $\beta_p = \frac{2^{B-p} (2^{p}-1)}{2^{B+1}-1}$. 
This result implies the fact that the number of pulls for any arm should be capped by $\Theta( (TB)^{2\beta_p}K^{-2\beta_p})$ during the $p$-th pass, which is the key idea for our multi-pass algorithm in Section \ref{sec:upper}. 

\section{Instance-Dependent Lower Bound}
\label{sec:instance}
We provide the instance-dependent lower bound as below. 
\begin{theorem}
\label{the:instance}
    Given a time horizon $T$, a stream of $K$ arms, and passes $1 \leq B \leq \log T$ over this stream, any algorithm $\mathcal{A}$ that can only store $o(K/B)$ arms in memory suffers regret
    \begin{equation*}
        \textnormal{Regret}(\mathcal{A}) = \Omega \left(  T^{\frac{1}{B+1}} \sum_{\Delta_x > 0}\frac{\mu^*}{\Delta_x}\right).
    \end{equation*}
    in expectation for some problem instances. 
\end{theorem}
The hard distribution for the instance-dependent lower bound also comes from Algorithm \ref{alg:construct}, and the details of the proof are deferred in Appendix \ref{app:instance}. 
This is the first result for instance-dependent lower bound when streaming actions and sublinear memory are permitted. 
Compared with the upper bound in \cite{DBLP:conf/colt/AgarwalKP22}, there is still a $O\left(B\sum_{\Delta_x > 0}\frac{\mu^*}{\Delta_x}\right)$ gap for future work. 

\section{Upper Bound}
\label{sec:upper}
In this section, we provide a worst-case upper bound for the stochastic streaming bandits with bounded arm memory and multi-pass. 
\begin{theorem}
\label{the:upper}
    Given a time horizon $T$, a stream of $K$ arms, and passes $B$ over this stream, there is an algorithm (Algorithm \ref{alg:ud}) that can only store $O(1)$ arms in memory and achieves regret
    \begin{equation*}
        \textnormal{Regret}(\mathcal{A}) = O \left( (TB)^{\frac{2^B}{2^{B+1}-1}} K^{\frac{2^B-1}{2^{B+1}-1}} 
\sqrt{\log T} \right)
\end{equation*}
with probability at least $1-1/\textnormal{poly}(T)$ for any problem instance. 
\end{theorem}

We defer the algorithm description and theoretical analysis to Appendix \ref{app:ub} due to the page limit. The proposed algorithm follows a `successive elimination-style' approach and employs a memory that can store only two arms: one for the best estimate arm and another for the newly read arm from the stream. The algorithm's key idea is to limit the number of pulls for each arm during the $p$-th pass, based on the theoretical analysis of the worst-case lower bound in Section \ref{sec:worst}. Specifically, the number of pulls for any arm is bounded by $\Theta((TB)^{2\beta_p}K^{-2\beta_p})$, where $\beta_p = \frac{2^{B-p}(2^{p}-1)}{2^{B+1}-1}$. Given a time horizon $T$, a stream of $K$ arms, and $B$ passes over the stream, the algorithm achieves near-optimal regret with only a logarithmic factor.

Compared to the result of $O(T^{\frac{2^B}{2^{B+1}-1}} \sqrt{KB \log T})$ in \cite{DBLP:conf/colt/AgarwalKP22}, our upper bound has a tighter dependence on $K$ and $B$. 
Note that the lower bounds hold for algorithms with $o(K/B)$ memory while our algorithm only store $2$ arms for the upper bound. 
This does not indicate the space gap, since that both previous work and our research show the sharp threshold phenomenon for streaming bandits. That is, with constant arm memory, one can achieve near-optimal regret, and increasing memory to any quality $M= o(K/B)$ has almost no impact on further reducing the regret.  
Note that there is also a logarithmic factor for the `successive elimination-style' algorithm in the classical centralized setting. 
We leave the optimal algorithm for the streaming bandits with multi-pass and bound arm memory as a future question, while previous work \cite{DBLP:journals/jmlr/AudibertB10} shave off the factor and provide an optimal bound for the centralized setting. 

\section{Related work}
\label{sec:related}
\paragraph{Streaming Bandits.}
Multi-armed bandits is a widely used framework for modeling the trade-off between exploration and exploitation in online decision problems.
For the streaming setting, 
\cite{DBLP:conf/aistats/LiauSPY18} consider stochastic bandits with constant arm memory and propose an algorithm that achieves a $O(\log 1/\Delta)$ factor of optimal instance-dependent regret with $O(\log T)$ passes, where $\Delta$ is the gap between the optimal arm and second-optimal arm. 
\cite{DBLP:conf/aaai/ChaudhuriK20} study stochastic bandits with $M$ stored arms and provide an algorithm with regret $\Tilde{O}(KM + (K^{3/2}\sqrt{T})/M)$. 
The lower and upper bounds for the single-pass streaming bandits model are provided by
\cite{DBLP:journals/corr/abs-2012-05142}. 
In subsequent work, \cite{DBLP:conf/colt/AgarwalKP22} consider the dependence on the number of passes $B$, and show the $\Omega(4^{-B}T^{2^B/(2^{B+1}-1)})$ worst-case lower bound, and provide an algorithm achieving regret $\Tilde{\Theta}( T^{2^B/(2^{B+1}-1)} \sqrt{KB} )$. 
Concurrent to our work, \cite{DBLP:journals/corr/abs-2306-02208} establish tight regret bounds in the single-pass streaming bandits with sublinear arm memory. 

Note that \cite{DBLP:conf/colt/AgarwalKP22} is the most relevant to us, and we now discuss the primary difference in results and techniques. 
Our contribution lies in providing tighter upper and lower worst-case bounds for streaming bandits with multi-pass and bounded arm memory. In particular, our worst-case lower bound showcases the dependence on the number of arms $K$. Furthermore, we present the first instance-dependent lower bound on regret for streaming bandits. 
Regarding the techniques employed, \cite{DBLP:conf/colt/AgarwalKP22} demonstrate the challenge of trapping one specific arm under a nearly uniform distribution in a single pass. They construct a difficult distribution based on the "round elimination" principle, wherein the residual instance after each pass remains arduous. 
In contrast, our theoretical analysis relies on two problem reductions: the reduction from the regret minimization problem to a series of $\epsilon$-optimal arm identification problems, and the reduction from the $\epsilon$-optimal arm identification problem to the distributed distribution detection problem. 
Furthermore, we investigate the interplay between different passes that one pass can provide valuable information for all subsequent passes, resulting in reduced total regret rather than simply maintaining a hard instance.

\paragraph{Pure Exploration Setting.}
The study of streaming algorithms for pure exploration bandits setting with bounded arm memory was initiated by \cite{DBLP:conf/stoc/AssadiW20}. 
They provide an algorithm with only one arm memory achieving optimal worst-case sample complexity $O(K\Delta^{-2})$ under the assumption that the gap between the expected of the optimal and second optimal arms $\Delta$ is known. 
They also show the top-$m$ arm identification problem requires $\Theta(m)$ arms of space. 
Subsequent work \cite{DBLP:conf/icml/JinH0X21} establish the near-optimal upper bound of the instance-sensitive sample complexity with $O(\log(1/\Delta))$ passes for the optimal arm identification problem. 
Very recently, \cite{assadi2022singlepass} provide both the single-pass upper and lower bounds for optimal arm identification with different conditions. 
Another line relevant to our work is the streaming $\epsilon$-optimal arm identification problem. 
A number of efforts \cite{DBLP:conf/stoc/AssadiW20, DBLP:journals/corr/abs-2012-05142, DBLP:conf/icml/JinH0X21} have been devoted to this setting and seek to design streaming algorithm with constant arm memory achieving sample complexity $O(n/\epsilon^2)$. 
However, little theoretical research has been conducted on the lower bound of the sample complexity for this setting. 

\paragraph{Learning in Streams.}
In addition to the bandits setting, a substantial amount of work has been achieved  on other learning problems. 
For another fundamental problem of sequential prediction, online learning with expert advice, \cite{DBLP:conf/stoc/SrinivasWXZ22}
initiates the study of memory complexity of the streaming setting. 
Subsequent work \cite{DBLP:journals/corr/abs-2207-07974} explores the single-pass algorithm with sub-linear space and sub-linear regret for the streaming expert problem against an oblivious adversary. 
The memory-efficient learning problem was solved with different situations, including statistical learning \cite{DBLP:conf/colt/SteinhardtVW16,garg2017extractor,DBLP:conf/focs/Raz17,DBLP:conf/coco/GargRT19,DBLP:conf/stoc/SharanSV19,DBLP:conf/colt/MarsdenSSV22}, estimation problems \cite{DBLP:conf/nips/AcharyaBIS19,DBLP:conf/icml/DiakonikolasKPP22,DBLP:conf/colt/BergOS22}, parity learning \cite{DBLP:journals/jacm/Raz19,DBLP:conf/stoc/KolRT17}, and other learning problems \cite{DBLP:conf/colt/HopkinsKLM21,DBLP:conf/colt/BrownBS22}. 


\section{Discussion and Conclusion}
\label{sec:conclusion}
In this work, we study the stochastic streaming bandits with multi-pass and constraint arm memory and provide the tight worst-case regret lower bound and instance-dependent lower bound. 
We show that any algorithm suffers $\Omega \left( (TB)^{\alpha} K^{1-\alpha}\right),\alpha = 2^{B} / (2^{B+1}-1)$ worst-case regret in expectation, with a time horizon $T$, number of arms $K$, number of passes $B$, and bounded arm memory $M = o(K/B)$. 
We also provide a stronger lower bound $\Omega \left(  T^{\frac{1}{B+1}} \sum_{\Delta_x > 0}\frac{\mu^*}{\Delta_x}\right)$ which guarantees "high" regret by taking advantage of `nice' problem instance. 
These results show the separation for stochastic multi-armed bandits problem between the classical centralized setting and the streaming setting. 
To complement the lower bounds, we provide a near-optimal upper bound up to a logarithmic factor. 
Note that there remains open problems for future research, specifically, the development of an optimal algorithm and a tight upper bound for streaming bandits with sublinear memory. 
And exploring the gap between the instance-dependent upper and lower bounds stands as another important avenue for future investigation. 

\newpage

\bibliography{mybib}
\bibliographystyle{alpha}

\newpage

\appendix

\section{Proof of Worst-case Lower Bound}
\label{app:lb}
In this section, we provide the details of proofs in Section \ref{sec:worst}. 

\subsection{Proof of Lemma \ref{lem:distri}}
\label{subsec:distri}
Let us revisit the distributed detection problem. 
We first introduce the definition of SDPI constant below. 
\begin{definition}[SDPI constant]
    Let $V \sim B_{0.5}$ (fair coin) and the channel $V \rightarrow X$ be defined as it is in the distributed detection problem. Then, there exists a constant $\beta \leq 1$ that depend on $\mu_0$ and $\mu_1$, s.t. for any transcript $\Pi$ that depends only on $X$ i.e. $V \rightarrow X \rightarrow \Pi$ is a Markov chain, we have
    \begin{equation}
    \label{eq:sdpi constant}
        I(V; \Pi) \leq \beta I(X;\Pi).
    \end{equation}
    Let $\beta(\mu_0, \mu_1)$, the SDPI constant, be the infimum over all possible $\beta$ such that \eqref{eq:sdpi constant} holds. 
\end{definition}
Note that the transcript $\Pi \in \{0, 1\}^*$ represents the messages of all communication. 
Consequently, the length of the transcript, denoted by $|\Pi|$, corresponds to the communication complexity of the protocol. Furthermore, it serves as a lower bound for the sample complexity, as at least one bit needs to be transferred through the channel for each sample.

From the form of equation \eqref{eq:sdpi constant}, the LHS corresponds to the accuracy of any algorithm, while the RHS provides the lower bound of sample complexity. 
For the SDPI constant of the distribution detection problem, according to Lemma 7 of \cite{DBLP:conf/nips/ZhangDJW13}, we have the following fact. 
\begin{fact}
\label{fact:1}
    Assume we sample $V$ from a distribution, and suppose $\sup_x \frac{\mathbb{P}_{X_i \sim \mu_{1}}(X_i = x_i)}{\mathbb{P}_{X_i \sim \mu_{0}}(X_i = x_i)} \leq c$ for all $i \in [T]$. Then, the following inequality holds:
    \begin{equation*}
        I(V;\Pi) \leq 2(c^2-1)^2 I(X;\Pi).
    \end{equation*}
\end{fact}
For the distributed detection problem, assume we sample $V$ from a distribution. 
Suppose $\frac{1}{2} \leq \mu_0 < \mu_1 \leq \frac{1}{2} + c$ for some constant $0 < c <\frac{1}{2}$ and let $\Delta: \mu_1 - \mu_0$ be their gap. 
Fact \ref{fact:1} implies the following inequality:
\begin{equation*}
    \frac{I(V;\Pi)}{I(X;\Pi)} \leq O(\Delta^2). 
\end{equation*}
Then we provide the Fano's inequality, which establishes a connection between the average information loss in a noisy channel and the probability of a categorization error.
\begin{fact}[Fano's inequality]
    Let $X$ and $Y$ be two random variables, correlated in general, with alphabets $\mathcal{X}$ and $\mathcal{Y}$, respectively, where $\mathcal{X}$ is finite but $\mathcal{Y}$ can be countably infinite. 
    Let $\hat{X}:= g(Y)$ be an estimate of $X$ from observing $Y$, where $g:\mathcal{Y} \rightarrow \mathcal{X}$ is a given estimation function. 
    Define the probability of error as $P_e:= \mathbb{P}(X \neq \hat{X})$. 
    then the following inequality holds
    \begin{equation*}
        H(X|Y) \leq H(P_e) + P_e \cdot \log(|\mathcal{X}|- 1). 
    \end{equation*}
\end{fact}
From the facts that $I(X;\Pi) \leq O(n)$ and $I(V;\hat{V}) \leq I(V;\Pi)$, we obtain
\begin{equation*}
    I(V;\hat{V}) \leq O(n\Delta^2)
\end{equation*}
Applying the Fano's inequality to the above equation, we complete the proof. 

\paragraph{Discussion.}
Note that the sample complexity result can also be derived from the findings presented in \cite{DBLP:conf/colt/Agarwal0AK17}, which relies on the insights gained from KL-divergence arguments. In contrast, our proof is based on the analysis of the distributed data processing inequality as demonstrated in \cite{DBLP:conf/stoc/BravermanGMNW16}. 
Lemma \ref{lem:distri} not only offers a novel version of the sample complexity for the coin tossing problem but also provides the communication bound for the distributed setting, which may serve as a valuable technique tools for the distributed bandits problem.

\subsection{Proof of Lemma \ref{lem:dist}}
\label{subsec:dist}
The proof follows a similar analysis for the $\epsilon$-\textsc{DiffDist} problem presented in \cite{DBLP:conf/stoc/SrinivasWXZ22}, and we provide the details below for completeness. 
We first establish the reduction from the distributed detection problem to the $\epsilon$-optimal detection setting below. 
Consider the independent and identically distributed (i.i.d.) samples from the distribution $\mu_0$ for all $i \neq i^*$ and let $\Tilde{\boldsymbol{X}}_i$ denote it, while we use $\Tilde{\boldsymbol{X}}_{i^*}$  to represent the samples of the $i^*$-th arm. 
Let $\Tilde{\boldsymbol{X}}$ be the union of all samples above. 
Then we consider an algorithm $\mathcal{A}$ that solves the $\epsilon$-optimal detection problem of the bandits setting, and any algorithm $\Tilde{\mathcal{A}}$ that solves the distributed distribution detection problem, respectively. 
The algorithm $\mathcal{A}$ inputs samples $\Tilde{\boldsymbol{X}}$, and outputs $V=1$ if and only if $\Tilde{\mathcal{A}}$ outputs $V=1$. 
Note that $\Tilde{\boldsymbol{X}}$ has to use at least $\Omega(\Delta^{-2})$ samples to solve the distributed distribution detection problem with probability $1-p$ for some constant $p< \frac{1}{2}$, according to Lemma \ref{lem:distri}. 

Then we restrict our attention to the $V=0$ case. 
We use $(\boldsymbol{X}, \Pi(\boldsymbol{X}))$ to denote the joint distribution of samples $\boldsymbol{X}$ and transcript $\Pi$.
Note that the distribution of $(\Tilde{\boldsymbol{X}}, \Pi(\Tilde{\boldsymbol{X}}))$ and the distribution $(\boldsymbol{X}, \Pi(\boldsymbol{X}))$ are equal. 
Then we have $I(\boldsymbol{X}_i;\Pi(\boldsymbol{X})) = I(\Tilde{\boldsymbol{X}}_i; \Pi(\Tilde{\boldsymbol{X}}))$. 
According to the dependence between different samples and different arms, we have $I(\boldsymbol{X};\Pi(\boldsymbol{X})) \geq \sum_{i \in [K]} I(\boldsymbol{X}_i;\Pi(\boldsymbol{X}))$. 
Similar to the proof in Lemma \ref{lem:distri}, $\Omega(\epsilon^{-2})$ samples are necessary to identify one arm. In other words, we have $I(\boldsymbol{X}_i; \Pi(\boldsymbol{X})) = \Omega(\epsilon^{-2})$. Combining it with $I(\boldsymbol{X};\Pi(\boldsymbol{X})) \geq \sum_{i \in [K]} I(\boldsymbol{X}_i;\Pi(\boldsymbol{X}))$, we obtain the result.

\subsection{Proof of Lemma \ref{lem:dis}}
\label{subsec:dis}
Note that the distribution for the first $K-1$ arms is equal to the distribution in Lemma \ref{lem:dist}. 
Observe that any algorithm that solves $\mathcal{D}^{(\mathtt{base},\epsilon)}_{\mathcal{K}}$ with probability $1-q$, it must solve the $\epsilon$-optimal arm detection problem with probability at least $1-\frac{q}{2}$, where $q$ is a constant satisfying $q \in [0, \frac{1}{2})$. 
Based on the reduction discussed in Section \ref{subsec:dist}, we are aware that any algorithm must utilize a minimum of $\Omega(\epsilon^{-2})$ pulls for each arm in a set comprising at least $\Omega(K)$ arms. 
Then we complete the proof. 
The result holds when the arm memory is $K$. In other words, there maybe not exists a sample complexity gap between the centralized setting and single-pass streaming setting for the $\epsilon$-optimal arm pure exploration scenario. However, the result still provides sufficient insight for analyzing the regret minimization setting of multi-pass case with sublinear arm memory.

\subsection{Proof of Theorem \ref{the:worst}}
\label{subsec:worst}

Given a time horizon $T$ and $K$ arms, let $\mathcal{I}$ denote the set of instances and $\mathcal{A}^{(\mathcal{I},\epsilon)}$ denote the algorithm satisfying 
\begin{equation*}
    \max_{I \in \mathcal{I}} \textnormal{Regret}(\mathcal{A}) \leq O(\epsilon T)
\end{equation*}
with probability at least $1-o(1)$. 
We refer to this event as the `clean event'.
Then for any algorithm $\mathcal{A}^{(\mathcal{I},\epsilon)}$, we have the following result.

\begin{lemma}
\label{lem:observe}
    For any algorithm $\mathcal{A}^{(\mathcal{I},\epsilon)}$ as defined above, there is a reduction algorithm based on the pulling history of $\mathcal{A}^{(\mathcal{I},\epsilon)}$, that can output at least one $O(\epsilon)$-optimal arm at the end of the game with a probability of $1-o(1)$ for any problem instance $I \in \mathcal{I}$, without any additional samples and information. 
\end{lemma}
\begin{proof}
    For an algorithm $\mathcal{A}^{(\mathcal{I},\epsilon)}$, let $\{x_1, x_2, \ldots, x_T\}$ represent the set of arms in the pulling history. The reduction algorithm random selects and outputs one specific arm, denoted as $x_{\epsilon}$, from this set of $T$ arms. If the arm $x_{\epsilon}$ is indeed an $\omega(\epsilon)$-optimal arm with a probability of $\Omega(1)$, then algorithm $\mathcal{A}^{(\mathcal{I},\epsilon)}$ selects the $\omega(\epsilon)$-optimal arm at least $\Omega(T)$ times. Consequently, this results in a regret of $\omega(\epsilon T)$ with a probability of $\Omega(1)$. 
    According to the definition of $\mathcal{A}^{(\mathcal{I},\epsilon)}$, we have $\mathbb{P}[\max_{I \in \mathcal{I}} \textnormal{Regret}(\mathcal{A}) \leq O(\epsilon T)] \geq 1 - o(1)$, which leads to the contradiction. Then we complete the proof. 
\end{proof}
According to Lemma \ref{lem:observe}, any algorithm $\mathcal{A}^{(\mathcal{I},\epsilon)}$ needs to collect more information than one $O(\epsilon)$-optimal arm identification algorithm. 
Then we first focus on the set of $O(\epsilon)$-optimal arm identification algorithms under the clean event below. 

\begin{lemma}
\label{lem:no knowledge}
    Given a time horizon $T$, a single-pass stream of $K$ arms, consider any single-pass $O(\epsilon)$-optimal arm identification algorithm $\mathcal{A}$ that only store $o(K)$ arms in memory, the incurred regret is at least $\Omega(\sum_{x \in \mathcal{K}} \Delta_x \epsilon^{-2}  )$ for some problem instances.
\end{lemma}

\begin{proof}
    Let us fixed the parameter $\epsilon'$ satisfying $\epsilon' = O(\epsilon)$, then consider any single-pass $\epsilon'$-optimal arm identification problem. 
    For the hard distribution $\mathcal{D}^{(\mathtt{base},\epsilon')}_{\mathcal{K}}$, we know any algorithm has to use $\Omega(\Delta_x^{-2})$ pulls for each of an arm set with $\Omega(K)$ arms according to the reduction in Section \ref{subsec:dist} and \ref{subsec:dis}. 
    Let $\mathcal{K}_{\epsilon'}$ denote the set of these $\Omega(K)$ arms. 
    Then the incurred regret is at least
    \begin{equation*}
        \sum_{x \in \mathcal{K}_{\epsilon'}} \Delta_x \cdot \Delta_x^{-2} \geq \Omega \left(  \sum_{x \in \mathcal{K}_{\epsilon'}} \Delta_x (\epsilon')^{-2} \right) \geq \Omega \left( \sum_{x \in \mathcal{K}} \Delta_x (\epsilon')^{-2}\right). 
    \end{equation*}
    The first inequality arises from the fact that we have $\Delta_x \in [\epsilon', 2\epsilon']$ for suboptimal arms in $\mathcal{D}^{(\mathtt{base},\epsilon')}_{\mathcal{K}}$, while the second one is derived from the fact that $|\mathcal{K}_{\epsilon'}| = \Omega(|\mathcal{K}|)$ according to Lemma \ref{lem:dis}. 
    Thus any single-pass $\epsilon'$-optimal arm identification problem incurs regret $\Omega(\sum_{x \in \mathcal{K}} \Delta_x (\epsilon')^{-2}  )$ for some problem instances. 
    Then we complete the proof by applying the fact $\epsilon' = O(\epsilon)$. 
\end{proof}
For the set of $O(\epsilon)$-optimal arm identification algorithms, we would like to demonstrate that selecting different $\epsilon$ parameters is meaningless and wasteful for obtaining information of the optimal arm. 
Consider the scenario in which the single-pass algorithm uses $\epsilon_1^{-2}$ pulls for all arms $x \in \mathcal{K}_1$, whereas uses $\epsilon_2^{-2}$ pulls for all arms $x \in \mathcal{K}_2$:
\begin{enumerate}
    \item If $\epsilon_1 = \Theta(\epsilon)$ and $\epsilon_2 = \Theta(\epsilon)$, then we have $|\mathcal{K}_1 \cup \mathcal{K}_2| = \Theta(K)$ and the algorithm has to use $\Omega(\epsilon^{-2})$ pulls for each of $\Theta(K)$ number of arms identify at least one $\epsilon$-optimal arm with constant probability. 
    \item If $\epsilon_1 = \omega(\epsilon)$ and $\epsilon_2 = \omega(\epsilon)$, then the exploration is insufficient, and the algorithm cannot identify at least one $\epsilon$-optimal arm with constant probability. 
    \item If $\epsilon_1 = \omega(\epsilon)$ and $\epsilon_2 = \Theta(\epsilon)$, then we have $|\mathcal{K}_2| = \Theta(K)$ and the algorithm has to use $\Omega(\epsilon^{-2})$ pulls for each of $\Theta(K)$ number of arms. 
    We have similar result for the case with $\epsilon_2 = \omega(\epsilon)$ and $\epsilon_1 = \Theta(\epsilon)$. 
\end{enumerate}
This thought experiment could be extended to cases with multiple parameters.
In the remainder of this section, we will only consider algorithms that select one parameter per pass.
Then we going to focus on the interactions between different passes and start with the two-pass case. 
Consider the single-pass $\epsilon$-optimal arm identification algorithm with enough knowledge about one $\epsilon'$-optimal arm, we have the following result. Note that we assume $\epsilon' = \omega( \epsilon)$ while the case of $\epsilon' =O( \epsilon)$ is trivial. 
\begin{lemma}
\label{lem:yes knowledge}
    Given a time horizon $T$, a single-pass stream of $K$ arms, and $\Omega((\epsilon')^{-2})$ samples of one $\epsilon'$-optimal arm (with constant probability at least $1-q, q \in [0, 0.5)$), consider any single-pass $\epsilon$-optimal arm identification algorithm  $\mathcal{A}$ that only store $o(K)$ arms in memory, the incurred regret is at least $\Omega(\sum_{x \in \mathcal{K}^{(0,\epsilon')}}\Delta_x \epsilon^{-2} + \sum_{x \in \mathcal{K}^{(\epsilon', 1)}}\Delta_x (\epsilon')^{-2})$ for some problem instances.
\end{lemma}

\begin{proof}
    Consider a relaxed case, that the expected reward of one arm $\mu_x$ that satisfies $\Delta_x \in [\frac{1}{2}\epsilon', \epsilon']$ is given, and let $x'$ denote this arm.  
    We use Algorithm \ref{alg:construct} to construct the hard distribution $\mathcal{D}_1 \otimes \mathcal{D}_2$ with $\mathtt{base}_1 = \mu_{x'}$, $\epsilon_1 = \frac{1}{2} \epsilon'$, $\epsilon_2 = \epsilon$, and $|\mathcal{K}_1| = |\mathcal{K}_2|$. 
    For all arms $x \in \mathcal{K}\setminus (\mathcal{K}_1 \cup \mathcal{K}_2)$, set $\Delta_x = (\epsilon', 2\epsilon')$. 
    For arms $x \in \mathcal{K}^{(0,\Delta_{x'})}$, there is a hard distribution $\mathcal{D}_2$ such that any algorithm has to use $\Omega(\epsilon^{-2})$ pulls for each of an arm set with $\Omega(|\mathcal{K}_2|)$ arms. 
    The incurred regret is
    $\Omega(\sum_{x \in \mathcal{K}^{(0,\epsilon')}}\Delta_x \epsilon^{-2})$ from the facts that $\mathcal{K}_2 \subseteq \mathcal{K}^{(0,\epsilon')}$ and $|\mathcal{K}_2| = \Omega(|\mathcal{K}^{(0,\epsilon')}|)$.For all arms $x \in \mathcal{K}\setminus (\mathcal{K}_1 \cup \mathcal{K}_2)$, with the expected reward of arm $x'$, 
    any algorithm still has to use $\Omega((\epsilon')^{-2})$ samples to identify the sub-optimality, which leads to the regret $\Omega(\sum_{x \in \mathcal{K}^{(\epsilon', 1)}}\Delta_x (\epsilon')^{-2})$. 
    Then we complete the proof. 
\end{proof}

Note that the knowledge could be collected by a single-pass $\epsilon'$-optimal arm identification algorithm. 
Another leading result is that any algorithm that fails to gather enough information about the $\epsilon'$-optimal arm during the $\epsilon'$-pass will require much more samples during the subsequent phases (for example, $\epsilon$-pass), due to the insufficient knowledge that leaked from the previous passes.
Then we obtain a two-pass algorithm $\mathcal{A}_{\textnormal{two-pass}}$ and the learner has to select parameters $\epsilon$ and $\epsilon'$ carefully to minimize the total regret. 
The exploration regret satisfies the following inequality with constant $c_1$:
\begin{equation*}
\vspace{-0.05in}
    \textnormal{Explore-Regret}(\mathcal{A}_{\textnormal{two-pass}}) \geq c_1 \cdot \Big( \underbrace{ \sum_{x \in \mathcal{K}}\Delta_x (\epsilon')^{-2} }_{\epsilon'-\textnormal{pass}} +  \underbrace{ \sum_{x \in \mathcal{K}^{(0,\epsilon')}}\Delta_x \epsilon^{-2} + \sum_{x \in \mathcal{K}^{(\epsilon', 1)}}\Delta_x (\epsilon')^{-2} }_{\epsilon-\textnormal{pass}}  \Big).
\end{equation*}
\vspace{-0.05in}

Then we extend the result to the multi-pass case. 
Let the $p$-th pass be with parameter $\epsilon_p$, i.e., aims to identify at least one $\epsilon_p$-optimal arm during the $p$-th pass. 
Similar to the two-pass case, the case with $\epsilon_{p+1} = \Omega(\epsilon_{p})$ is trivial for all $p \in [B-1]$. 
This also implies that with enough knowledge about the $\epsilon_p$-optimal arm, the $\epsilon_{p+1}$-optimal arm detection setting, even the detection setting is still hard for the learner. 
For the regret within the $\epsilon_{p+1}$-pass, we have the following inequality
\vspace{-0.05in}
\begin{equation*}
    \max_{\Delta_x, x \in \mathcal{K}} \Big( \sum_{x \in \mathcal{K}^{(0,\epsilon_p)}}\Delta_x \epsilon_{p+1}^{-2} + \sum_{x \in \mathcal{K}^{(\epsilon_p, 1)}}\Delta_x (\epsilon_p)^{-2} \Big) \geq \max_{\Delta_x, x \in \mathcal{K}^{(0,\epsilon_p)}} \sum_{x}\Delta_x \epsilon_{p+1}^{-2} \geq \frac{1}{2} |\mathcal{K}^{(\epsilon_{p+1}, \epsilon_p)}| \epsilon_p \epsilon_{p+1}^{-2}. 
\end{equation*}
In other words, even with the knowledge from the previous passes, $\Omega \left( \sum_x \epsilon_p \epsilon_{p+1}^{-2} \right)$ regret is inevitable for the $\epsilon_{p+1}$-pass while the adversary selects the hard distribution carefully.

Now we provide the hard distribution for the multi-passes case. 
Note that an algorithm with $B$ passes and $M$ bounded arm memory could be denoted as $\langle B, M, \epsilon_1, \epsilon_2, \ldots, \epsilon_B \rangle$. 
Let $\mathcal{K}_1, \mathcal{K}_2, \ldots, \mathcal{K}_B$ be a division of the arm set $\mathcal{K}$. 
Given parameters $\epsilon_1, \epsilon_2, \ldots, \epsilon_B$, we provide an algorithm (shown in Algorithm \ref{alg:construct}) to construct hard distribution. 
Let $\epsilon_0 = \frac{1}{4}$, then the regret of any algorithm satisfies
\vspace{-0.05in}
\begin{equation*}
    \textnormal{Regret}(\mathcal{A}) \geq \frac{c_1}{2} \cdot \mathbb{E} \Bigg[  \min_{\epsilon_1, \epsilon_2, \ldots, \epsilon_B} \max_{\mathcal{K}_1, \mathcal{K}_2, \ldots, \mathcal{K}_B} \underbrace{  \sum_{p = 1}^{B} \frac{\epsilon_{p-1} |\mathcal{K}^{(\epsilon_p, \epsilon_{p-1})}| }{\epsilon_{p}^2} }_{\textnormal{Exploration part}} + \underbrace{ \epsilon_B \left( T - \sum_{p = 1}^{B} \frac{1}{\epsilon^2_p}\right) }_{\textnormal{Exploitation part}} \Bigg]. 
\end{equation*}

For the hard distribution in Algorithm \ref{alg:construct}, the optimal arm falls within the set $\mathcal{K}_p$ with probability at least $1/2^p$. 
In other words, there is at least one hard distribution such that $\mathbb{E}(|\mathcal{K}^{(\epsilon_p, \epsilon_{p-1})}|) \geq \Omega( |\mathcal{K}_p|)$ for all $p \in [B]$. 
Let $|\mathcal{K}_p| = K/B$, we obtain
\vspace{-0.05in}
\begin{equation}
\label{eq:regret}
\begin{aligned}
    \textnormal{Regret}(\mathcal{A}) &\geq \frac{c_1}{2} \min_{\epsilon_1, \epsilon_2, \ldots, \epsilon_B} \sum_{p = 1}^{B}  \frac{\epsilon_{p-1}  } {\epsilon_{p}^2} \cdot \frac{K}{B}+ \epsilon_B \left( T - \sum_{p = 1}^{B} \frac{1}{\epsilon^2_p}\right)\\
    &\geq \frac{c_1}{4} \min_{\epsilon_1, \epsilon_2, \ldots, \epsilon_B} \sum_{p = 1}^{B}  \frac{\epsilon_{p-1}  } {\epsilon_{p}^2} \cdot \frac{K}{B}+ \epsilon_B T. 
\end{aligned}
\end{equation}
The second inequality comes from the fact that we have $\epsilon_p = \omega(\epsilon_B)$ for all $p \in [B-1]$. 
Let $f:[0, \frac{1}{4}]^B \rightarrow \mathbb{R}_+$ be the function of $\epsilon_p, p \in [B]$:
\begin{equation*}
    f(\epsilon_1, \epsilon_2, \ldots, \epsilon_B) = \sum_{p = 1}^{B}  \frac{\epsilon_{p-1}  } {\epsilon_{p}^2} \cdot \frac{K}{B}+ \epsilon_B T. 
\end{equation*}
By taking $\frac{\partial f }{\partial \epsilon_1} = \frac{\partial f }{\partial \epsilon_2} = \ldots = \frac{\partial f }{\partial \epsilon_B} = 0$ and let $c_2$ be a constant, we obtain the following result:
\begin{equation}
\label{eq:func}
    f(\epsilon_1, \epsilon_2, \ldots, \epsilon_B) \geq c_2 \cdot (TB)^{\frac{2^B}{2^{B+1}-1}} K^{\frac{2^B-1}{2^{B+1}-1}}. 
\end{equation}
Combine \eqref{eq:regret} and \eqref{eq:func} together, we obtain
\begin{equation}
\label{eq:worst case}
    \textnormal{Regret}(\mathcal{A}) = \Omega \left( (TB)^{\frac{2^B}{2^{B+1}-1}} K^{\frac{2^B-1}{2^{B+1}-1}} \right). 
\end{equation}
Then, let us consider algorithms that deviate from the objective of finding an $\epsilon_p$-optimal arm in each pass $p$. We will begin with the $B$-th pass. According to Lemma \ref{lem:observe} and the reduction technique, any algorithm $\mathcal{A}^{(\mathcal{I},\epsilon_B)}$ must gather more information compared to an $O(\epsilon_B)$-optimal arm identification algorithm. Consequently, any algorithm that fails to identify even a single $O(\epsilon_B)$-optimal arm in pass $B$ will suffer a higher regret. 
If an algorithm successfully identifies one $O(\epsilon_B)$-optimal arm during pass $B$ but fails to identify an $O(\epsilon_{B-1})$-optimal arm during the $(B-1)$-th pass, it will suffer more regret during the $B$-th pass, as indicated by Lemma \ref{lem:no knowledge} and \ref{lem:yes knowledge}. We can continue this thought experiment until the $1$-st pass, demonstrating that any algorithm $\mathcal{A}$ attempting an approach different from finding an $\epsilon_p$-optimal arm in each pass $p$ will inevitably incur a regret given by \eqref{eq:worst case}. 
Combine \eqref{eq:worst case} with the union probability of clean events for different phase, then we obtain the expected regret lower bound. 

\section{Proof of Instance-dependent Lower Bound}
\label{app:instance}
The hard distribution for the instance-dependent lower bound also comes from Algorithm \ref{alg:construct}, and the theoretical analysis is a bit different. 
Revisit the regret within the $\epsilon_{p+1}$-pass, we have the following inequality:
\begin{equation*}
\begin{aligned}
    &\max_{\Delta_x, x \in \mathcal{K}} \Big( \sum_{x \in \mathcal{K}^{(0,\epsilon_p)}}\Delta_x \epsilon_{p+1}^{-2} + \sum_{x \in \mathcal{K}^{(\epsilon_p, 1)}}\Delta_x (\epsilon_p)^{-2} \Big) \geq \max_{\Delta_x, x \in \mathcal{K}^{(0,\epsilon_p)}} \sum_{x}\Delta_x \epsilon_{p+1}^{-2} \\
    &\geq \max_{\Delta_x, x \in \mathcal{K}^{(0,\epsilon_p)}} \sum_{x} \frac{1}{\Delta_x} \cdot \frac{\Delta_x^{2}}{\epsilon_{p+1}^{2}} \geq \frac{\epsilon_p}{2 \epsilon_{p+1}}  \sum_{x \in \mathcal{K}^{(\epsilon_{p+1}, \epsilon_p)}}\frac{1}{\Delta_x}. 
\end{aligned}
\end{equation*}
Then the regret of any algorithm satisfies
\begin{equation}
\label{eq:instance}
\begin{aligned}
    \textnormal{Regret}(\mathcal{A}) &\geq \frac{c_1}{4} \cdot \mathbb{E} \left[  \min_{\epsilon_1, \epsilon_2, \ldots, \epsilon_B} \max_{\mathcal{K}_1, \mathcal{K}_2, \ldots, \mathcal{K}_B}   \sum_{p = 1}^{B}  \frac{\epsilon_{p-1}}{\epsilon_{p}}  \sum_{x \in \mathcal{K}^{(\epsilon_{p+1}, \epsilon_p)}}\frac{1}{\Delta_x}  +  \epsilon_B \left( T - \sum_{p = 1}^{B} \frac{1}{\epsilon^2_p}\right)  \right]\\
    &\geq \frac{c_1}{8} \cdot \mathbb{E} \left[  \min_{\epsilon_1, \epsilon_2, \ldots, \epsilon_B} \max_{\mathcal{K}_1, \mathcal{K}_2, \ldots, \mathcal{K}_B}   \sum_{p = 1}^{B}  \frac{\epsilon_{p-1}}{\epsilon_{p}}  \sum_{x \in \mathcal{K}^{(\epsilon_{p+1}, \epsilon_p)}}\frac{1}{\Delta_x}  +  \epsilon_B T  \right]\\
    &\geq \frac{c_1}{16}  \min_{\epsilon_1, \epsilon_2, \ldots, \epsilon_B} \left( \sum_{p = 1}^{B}  \frac{\epsilon_{p-1}}{\epsilon_{p}} + \epsilon_B T \right) \frac{1}{B+1}  \sum_{x \in \mathcal{K}, \Delta_x > 0}\frac{1}{\Delta_x}\\
    &\geq \Omega \left( T^{\frac{1}{B+1}} \sum_{x \in \mathcal{K}, \Delta_x > 0}\frac{\mu^*}{\Delta_x} \right).
\end{aligned}
\end{equation}
The second inequality of \eqref{eq:instance} comes from the fact $\epsilon_p = \omega(\epsilon_B)$ for all $p \in [B-1]$. 
The third inequality of \eqref{eq:instance} base on the observation that $\mathcal{K}^{(\epsilon_1, \epsilon_0)}, \mathcal{K}^{(\epsilon_2, \epsilon_1)}, \ldots, \mathcal{K}^{(\epsilon_{B-1}, \epsilon_B)}, \mathcal{K}^{(\Delta_{[2]}, \epsilon_B)}$ is a division of the sub-optimal arm set and taking 
\begin{equation*}
    \sum_{x \in \mathcal{K}^{(\epsilon_1, \epsilon_0)}} \frac{1}{\Delta_x} = \sum_{x \in \mathcal{K}^{(\epsilon_2, \epsilon_1)}} \frac{1}{\Delta_x} = \ldots = \sum_{x \in \mathcal{K}^{(\epsilon_{B-1}, \epsilon_B)}} \frac{1}{\Delta_x} = \sum_{x \in \mathcal{K}^{(\Delta_{[2]}, \epsilon_B)}} \frac{1}{\Delta_x}.
\end{equation*}
We obtain the last inequality of \eqref{eq:instance} by taking $\epsilon_0 \leq \mu^*$ and evaluating the minimum value of the function $g(\epsilon_1, \epsilon_2, \ldots, \epsilon_B) = \sum_{p = 1}^{B}  \frac{\epsilon_{p-1}}{\epsilon_{p}} + \epsilon_B T $. 
Similar to the discussion in Section \ref{subsec:worst}, algorithms that try to do something different than finding an $\epsilon_p$-optimal arm in each pass $p$ suffer larger regret. 
Combine \eqref{eq:instance} with the union probability of clean events for different phase, then we obtain the expected regret lower bound. 

\section{Algorithm and Proof of Upper Bound}
\label{app:ub}
This section provides the intuition, specification, and theoretical analysis of the Multi-Pass Successive Elimination algorithm (shown in Algorithm \ref{alg:ud}) for the stochastic streaming bandits problem. 

\subsection{Algorithm Description}

To ease the presentation, let us start with the main idea. 
The algorithm employed follows a `successive elimination-style' approach, which entails selecting a nearly optimal arm through exploration and subsequently utilizing the remaining rounds for exploitation. 
Throughout the exploration phase, the algorithm maintains two units of space: one for the best-estimated arm and another for the newly read arm. Note that the best-estimated arm serves not only as input for the exploitation phase but also aids in rapidly distinguishing sub-optimal arms.
The exploration phase is divided into $B$ passes. For each pass, the algorithm re-scans all the arms within the stream, with the objective of outputting a single optimal arm with a high probability upon completion of the final pass.

The key idea of the algorithm revolves around ensuring that the number of pulls for any given arm remains capped by $\Theta( (TB)^{2\beta_p}K^{-2\beta_p}), \beta_p = \frac{2^{B-p} (2^{p}-1)}{2^{B+1}-1}$ during the $p$-th pass. The constraint is derived from the theoretical analysis of the worst-case lower bound, as discussed in Section \ref{sec:worst}. 
The algorithm consistently selects the arm with the fewest pulls from the memory. Once a sufficient number of samples have been collected, it proceeds to eliminate a sub-optimal arm based on its upper confidence bound, subsequently reading a new arm into the memory.
Notice that the algorithm may prioritize two sub-optimal arms with a small reward gap, and the algorithm eliminates one arm from memory when the cap $\Theta( (TB)^{2\beta_p}K^{-2\beta_p})$ is achieved for both.

\begin{algorithm}[t]
    \caption{Multi-Pass Successive Elimination}
    \label{alg:ud}
    {\bfseries Input:} number of passes $B$, time horizon $T$
\begin{algorithmic}[1]
    \STATE $y \leftarrow x_1$, $\Bar{r}_y \leftarrow 0$, $n_y \leftarrow 0$
    \FOR{$p=1, 2, \cdots, B $}
    \STATE $x \leftarrow x_1$, $\Bar{r}_x \leftarrow 0$, $n_x \leftarrow 0$, $\beta_p \leftarrow \frac{2^{B-p} (2^{p}-1)}{2^{B+1}-1}$ ,  $b_p \leftarrow (TB)^{2\beta_p}K^{-2\beta_p}$
    \WHILE{pass is not finished}
    \WHILE{$n_x \leq b_p$ or $n_y \leq b_p$}
    \STATE Pull the least played arm between $x$ and $y$, and select a random arm if there not exists a least played arm
    \STATE Update $\Bar{r}_x$, $n_x$, $\Bar{r}_y$, $n_y$
    \IF{$ \min \{ \Bar{r}_x + \sqrt{(5 \log T)/n_x}, 1\} < \max \{ \Bar{r}_y - \sqrt{(5 \log T)/n_y}, 0\} $}
    \STATE Break
    \ELSIF{$ \max\{ \Bar{r}_x - \sqrt{(5 \log T)/n_x}, 0\} > \max \{ \Bar{r}_y - \sqrt{(5 \log T)/n_y}, 0\}$}
    \STATE $y \leftarrow x, \Bar{r}_y \leftarrow \Bar{r}_x, n_y \leftarrow n_x$
    \STATE Break
    \ENDIF
    \ENDWHILE
    \STATE Read a new arm into $x$ from the stream and update $\Bar{r}_x$, $n_x$. 
    \ENDWHILE
    \ENDFOR
    \STATE Play arm $y$ until the end of the game
\end{algorithmic}
\end{algorithm}

Now we present the pseudocode for the algorithm. 
The algorithm maintains three statistics for one arm in memory: the index $x$, the mean reward estimator $\Bar{r}_x$, and the number of pulls $n_x$. 
The exploration phase consists of $B$ phases, with each phase corresponding to a pass over the stream.
Let $b_p$ be the cap of samples for the $p$-th phase. 
We use $y$ and $x$ to denote the best-estimated arm and the new arm in the algorithm, respectively. 
For the algorithm, we have the following theoretical result.

\begin{theorem}
\label{the:uni}
    Given a time horizon $T$, a stream of $K$ arms, arm memory $M=2$, and passes $B$ over this stream, Algorithm \ref{alg:ud}  achieves regret
    \begin{equation*}
        \textnormal{Regret}(\mathcal{A}) = O \left( (TB)^{\frac{2^B}{2^{B+1}-1}} K^{\frac{2^B-1}{2^{B+1}-1}} 
\sqrt{\log T} \right)
    \end{equation*}
with probability at least $1-1/\textnormal{poly}(T)$ for any problem instance. 
\end{theorem}

\subsection{Theoretical Analysis}
For any fixed arm $x \in \mathcal{K}$, 
according to Hoeffding’s inequality, the gap between the expect reward $\mu(x)$ and the estimated mean reward $\Bar{r}_x$ can be bounded with a probability of at least $1 - T^{-5}$, as shown in the following inequality:
\begin{equation}
\label{eq:hoeff}
    |\mu(x) - \Bar{r}_x| \leq \sqrt{\frac{5 \log T}{n_x}}.
\end{equation}
By applying a union bound over all arms and all rounds, the inequality \eqref{eq:hoeff} holds for all arms $x_t \in \mathcal{K}$ and all rounds $t \in [T]$ with probability at least $1 - T^{-1}$. 
Let $\mathcal{E}$ denote this clean event, then we first analyze the regret based on this clean event.

Let $\mathbb{R}_{\mathcal{K}}^p$ denote the regret for the $p$-th phase. 
Since the number of pulls for all arms in phase $1$ is bounded by $K b_1$, we can establish the following inequality
\begin{equation}
\label{eq:phase_1}
\begin{aligned}
    \mathbb{R}_{\mathcal{K}}^1 \leq K \cdot (TB)^{2\beta_1}K^{-2\beta_1} = (TB)^{\alpha}K^{1-\alpha}.
\end{aligned}
\end{equation}
Then we focus on $\mathbb{R}_{\mathcal{K}}^p, p \in [B], p \geq 2$. 
Let $\mu^*_{\mathcal{K}} := \sup_{x \in \mathcal{K}} \mu(x)$ denote the expected per-round reward of the optimal arm in the set $\mathcal{K}$.  
We use $x^*_p$ and $\mu^*_p$ to denote the arm selected as optimal during phase $p$ and its corresponding expected per-round reward, respectively. 
For phase $p$, we consider the optimal estimated arm $y$ at the start of the $p$-th phase. 
If $x^*_{p-1}$ is discarded in phase $p-1$, according to the stop condition of compare strategy, we can derive the following inequality
\begin{equation*}
\begin{aligned}
    \Bar{r}_y &\geq \mu^*_{p-1} - \sqrt{\frac{5 \log T}{n_y}} - \sqrt{\frac{5 \log T}{n_{x^*_{p-1}}}}\\
    &\geq \mu^*_{p-1} - 2\sqrt{\frac{5 \log T}{b_{p-1}}}.
\end{aligned}
\end{equation*}

For arbitrary discarded arm $x$, we define $R_x^p$ and $N_x^p$ as the accumulated reward and total number of pulls during phase $p$, respectively. 
Notice that the value of $\Bar{r}_y - \sqrt{(5 \log T)/n_y}$ is non-decreasing, which  allows us to establish the following inequality
\begin{equation*}
\begin{aligned}
    \frac{R_x^p}{N_x^p - 1} + \sqrt{\frac{5 \log T}{N_x^p - 1}} \geq \mu^*_{p-1} - 2\sqrt{\frac{5 \log T}{b_{p-1}}}.
\end{aligned}
\end{equation*}
We obtain
\begin{equation*}
    R_x^p \geq 2N_x^p \left( \mu^*_{p-1} - \sqrt{\frac{5 \log T}{N_x^p - 1}} - \sqrt{\frac{5 \log T}{b_{p-1}}}  \right).
\end{equation*}
We use $\mathbb{R}_x^p$ to denote the cumulative regret of playing arm $x$ during phase $p$, then we have
\begin{equation*}
\begin{aligned}
    \mathbb{R}_x^p &\leq 2N_x^p \left(\sqrt{\frac{5 \log T}{N_x^p - 1}} + \sqrt{\frac{5 \log T}{b_{p-1}}} \right)  \\
    &\leq 2 \left( \sqrt{6 N_x^p \log T} + N_x^p\sqrt{\frac{5 \log T}{b_{p-1}}} \right).
\end{aligned}
\end{equation*}
The first term in the inequality accounts for the gap between the expected reward of the best-estimated arm and the selected sub-optimal arm. The second term represents the deviation between the best-estimated arm and the optimal expected per-round reward of the $(p-1)$-th phase. 
Let $\mathcal{K}_p$ denote the set of arms in phase $p$.
Applying Jensen's inequality, we have
\begin{equation*}
    \frac{1}{|\mathcal{K}_p|} \sum_{x \in \mathcal{K}_p} \sqrt{N^p_x} \leq \sqrt{\frac{1}{|\mathcal{K}_p|}\sum_{x \in \mathcal{K}_p}N^p_x} \leq \sqrt{ b_p }.
\end{equation*}
From this, we can derive
\begin{equation*}
    \sum_{x \in \mathcal{K}_p} \sqrt{N^p_x} \leq |\mathcal{K}_p| \sqrt{ b_p } \leq K \sqrt{b_p}.
\end{equation*}

Consider all selected arms during phase $p$ and the stop condition of the compare strategy, we have
\begin{equation*}
\begin{aligned}
    \mathbb{R}_{\mathcal{K}}^p &\leq 2 \sum_{x \in \mathcal{K}_p} \left( \sqrt{6 N_x^p \log T} + N_x^p\sqrt{\frac{5 \log T}{b_{p-1}}}  \right)\\
    &\leq 3K b_p \left(\sqrt{\frac{5 \log T}{b_{p-1}}}  \right) + 2 \sqrt{6\log T} \sum_{x \in \mathcal{K}_p}\sqrt{ N_x^p }\\
    &\leq 3K b_p \left(\sqrt{\frac{5 \log T}{b_{p-1}}}  \right) + 3 K \sqrt{ 6 b_p \log T}.
\end{aligned}
\end{equation*}
For the incurred regret by the deviation between the expected reward of best estimated arm and the selected sub-optimal arm, we have
\begin{equation}
\label{eq:phase_other1}
\begin{aligned}
    &\sum_{p=2}^{B} 3 K \sqrt{ 6 b_p \log T} \leq O( K \sqrt{ 6 b_B \log T} ) \leq O((TB)^{\alpha}K^{1-\alpha}\sqrt{\log T}).
\end{aligned}
\end{equation}

For the incurred regret the deviation between the best estimated arm and optimal expected per-round reward, we have
\begin{equation}
\label{eq:phase_other2}
\begin{aligned}
    \sum_{p=2}^{B} 3K b_p \left(\sqrt{\frac{5 \log T}{b_{p-1}}}  \right)  \leq O \left( K b_{B} \sqrt{\frac{5 \log T}{b_{B - 1}}} \right) \leq O((TB)^{\alpha}K^{1-\alpha} \sqrt{\log T}).
\end{aligned}
\end{equation}
Combine \eqref{eq:phase_1}, \eqref{eq:phase_other1}, and \eqref{eq:phase_other2} together, we can derive the incurring regret during the exploration part
\begin{equation}
\label{eq:exploration}
\begin{aligned}
    \sum_{p=1}^{B} \mathbb{R}_{\mathcal{K}}^p \leq  O((TB)^{\alpha}K^{1-\alpha} \sqrt{\log T}).
\end{aligned}
\end{equation}
Consider the selected arm $y$ after the exploration and let $\mathbb{R}_{\mathcal{K}}^y$ denote the regret due to selecting it. 
According to the stop condition of compare strategy, we have
\begin{equation*}
\begin{aligned}
    \Bar{r}_y \geq \mu^*_{\mathcal{K}} - 2\sqrt{\frac{5 \log T }{ b_{B}  }}.
\end{aligned}
\end{equation*}
Then for the regret during the exploitation phase, we obtain the incurred regret
\begin{equation}
\label{eq:exploitation}
\begin{aligned}
    \mathbb{R}_{\mathcal{K}}^y &\leq 2 T \sqrt{\frac{5 \log T }{ b_{B}  }} \leq O((TB)^{\alpha}K^{1-\alpha} \sqrt{\log T}).
\end{aligned}
\end{equation}
In the context of the clean event, combine the incurred regret during the exploration part \eqref{eq:exploration} and the incurred regret during the exploitation part \eqref{eq:exploitation} together, we obtain the total regret
\begin{equation*}
\begin{aligned}
    \mathbb{E}[ \mathbb{R}_{\mathcal{K}}(T) | \mathcal{E}] 
    &= \mathbb{R}_{\mathcal{K}}^y + \sum_{p=1}^{B} \mathbb{R}_{\mathcal{K}}^p \leq O((TB)^{\alpha}K^{1-\alpha} \sqrt{\log T}).
\end{aligned}
\end{equation*}
We obtain the regret
\begin{equation*}
\begin{aligned}
    \mathbb{R}_{\mathcal{K}}(T) &= \mathbb{E}[ \mathbb{R}_{\mathcal{K}}(T) | \mathcal{E}] \cdot \mathbb{P}(\mathcal{E}) + \mathbb{E}[ \mathbb{R}_{\mathcal{K}}(T) | \neg \mathcal{E}] \cdot \mathbb{P}(\neg \mathcal{E})\\
    &\leq O((TB)^{\alpha}K^{1-\alpha} \sqrt{\log T})(1 - 1/T) + 1\\
    &\leq O((TB)^{\alpha}K^{1-\alpha} \sqrt{\log T}).
\end{aligned}
\end{equation*}
Then we complete the theoretical analysis of Algorithm \ref{alg:ud}.



\end{document}